\documentclass[graybox]{svmult}

\usepackage{mathptmx,helvet,courier,makeidx,graphicx,multicol}
\usepackage[bottom]{footmisc}
\makeindex
\usepackage{amsfonts,amsmath,amssymb}
\usepackage{algorithm,algorithmic}

\newcommand{\reals}{\mathbb{R}}
\newcommand{\Exp}{\mathbb{E}}

\newcommand{\argmin}[1]{\underset{#1}{\mathrm{argmin}}}
\newcommand{\argmax}[1]{\underset{#1}{\mathrm{argmax}}}
\newcommand{\dt}{\displaystyle}

\newcommand{\boldf}{\mathbf{f}}

\newcommand{\by}{\mathbf{y}}

\newcommand{\bp}{\mathbf{p}}

\newcommand{\Ocal}{\mathcal{O}}

\newcommand{\Fcal}{\mathcal{F}}

\newcommand{\Rcal}{\mathcal{R}}

\newcommand{\Pcal}{\mathcal{P}}

\newcommand{\Vcal}{\mathcal{V}}
\newcommand{\Wcal}{\mathcal{W}}
\newcommand{\Ycal}{\mathcal{Y}}

\newcommand{\abs}{\mathrm{abs}}

\newcommand{\secref}[1]{Sec.~\ref{#1}}

\newcommand{\lemref}[1]{Lemma~\ref{#1}}

\newcommand{\thmref}[1]{Thm.~\ref{#1}}

\newcommand{\algref}[1]{Algorithm~\ref{#1}}

\begin{document}

\title*{Efficient Transductive Online Learning via Randomized Rounding}
\author{Nicol\`{o} Cesa-Bianchi and Ohad Shamir}
\institute{Nicol\`{o} Cesa-Bianchi \at DI, Universit\`{a} degli Studi di Milano, Italy, \email{nicolo.cesa-bianchi@unimi.it}
\and Ohad Shamir \at Microsoft Research, USA, \email{ohadsh@microsoft.com}}

\maketitle

\abstract{Most traditional online learning algorithms are based on variants of mirror descent or follow-the-leader. In this paper, we present an online algorithm based on a completely different approach, tailored for transductive settings, which combines ``random playout'' and randomized rounding of loss subgradients. As an application of our approach, we present the first computationally efficient online algorithm for collaborative filtering with trace-norm constrained matrices. As a second application, we solve an open question linking batch learning and transductive online learning.}

\section{Introduction}
Online learning algorithms, which have received much attention in recent years, enjoy an attractive combination of computational efficiency, lack of distributional assumptions, and strong theoretical guarantees. Informally speaking, online learning is framed as a sequential game between a \emph{learner}, who provides predictions, and an all-powerful \emph{adversary}, who chooses the outcomes on which the learner's predictions are tested. The learner's goal is to attain low regret ---that is, low excess loss--- with respect to a comparison class of experts or predictors (see \secref{sec:basic} for a more precise statement). Using standard online-to-batch techniques (e.g. \cite{CesCoGe04}), one can convert online learning methods into simple and effective batch learning algorithms in a stochastic setting, where training and test examples are sampled from a distribution.

In this work, we focus on transductive online learning, where the predictions of the experts/predictors can all be computed in advance. For example, consider the case where a sequence of unlabeled instances $\{x_t\}$ are given, and the learner needs to predict the corresponding labels $\{y_t\}$ which are sequentially chosen and revealed by the adversary. Thus, for a given fixed predictor $h$, we can already compute its predictions $\{h(x_t)\}$ beforehand. This is a natural online analogue of the transductive learning framework introduced by Vapnik in a statistical batch setting \cite{Vap98}, where the test instances on which one needs to predict are known in advance.

Despite the effectiveness of online learning methods, it is probably fair to say that at their core, most of them are based on the same small set of fundamental techniques, in particular mirror descent and regularized follow-the-leader (see for instance \cite{Hazan11,SS12}). In this work we revisit, and significantly extend, an algorithm which uses a completely different approach. This algorithm, known as the \emph{Minimax Forecaster}, was introduced in~\cite{CesabianchiFrHeHaScWa97,Chung94} for the setting of prediction with static experts. The Forecaster computes minimax predictions in the case of a fixed horizon, binary outcomes, and absolute loss. Although the original version is computationally expensive, it can easily be made efficient through randomization.

We extend the analysis of~\cite{CesabianchiFrHeHaScWa97} to the case of non-binary outcomes, unknown horizons, and arbitrary convex and Lipschitz loss functions. The new algorithm is based on a combination of ``random playout'' and randomized rounding, which assigns random binary labels to future unseen instances, in a way depending on the loss subgradients. Our resulting \emph{Randomized Rounding ($R^2$) Forecaster} has a parameter trading off regret performance and computational complexity, and runs in polynomial time. The idea of ``random playout'', in the context of online learning, has also been used in~\cite{AbWa10,KakadeKalai05}, but we apply this idea in a different way.

Interestingly, our work, which focuses on online learning, has close links to methods and concepts from statistical learning, and thus can be seen as bridging between the two fields. For example, the $R^2$ Forecaster uses empirical risk minimization ---a standard statistical learning method--- as a subroutine. Moreover, the regret of the $R^2$ Forecaster is determined by the Rademacher complexity of the comparison class, which is a measure of the generalization performance of the class in a statistical setting. The connection between online learnability and Rademacher complexity has also been explored in~\cite{RaSriTe10,AbBaRaTe09}. Recently, \cite{RaShaSri12} provided a significant generalization of these ideas, implying new algorithms and extending in a sense the work presented here.

As an application of our results, we describe how the $R^2$ Forecaster can be used to design the first efficient online learning algorithm for collaborative filtering with trace-norm constrained matrices. While this is a well-known setting, a straightforward application of standard online learning approaches, such as mirror descent, appear to give only trivial performance guarantees. Moreover, our regret bound matches the best known sample complexity bound in the batch distribution-free setting~\cite{ShalSham11}.

As a different application, we consider general reductions between batch learning and transductive online learning. The relationship between these two settings was analyzed in \cite{KakadeKalai05}, in the context of binary prediction with respect to classes of bounded VC dimension. Their main result was that efficient learning in a statistical setting implies efficient learning in the transductive online setting, but at an inferior rate of $T^{3/4}$ (where $T$ is the number of rounds). The main open question posed by that paper is whether a better rate can be obtained. Using the $R^2$ Forecaster, we improve on those results, and provide an efficient algorithm with the optimal $\sqrt{T}$ rate, for a wide class of losses. This shows that efficient batch learning not only implies efficient transductive online learning (the main thesis of \cite{KakadeKalai05}), but also that the same rates can be obtained, and for possibly non-binary prediction problems as well.

We emphasize that the $R^2$ Forecaster requires computing many empirical risk minimizers (ERM's) at each round, which might be prohibitive in practice. Thus, while it does run in polynomial time whenever an ERM can be efficiently computed, we make no claim that it is a practical algorithm. Nevertheless, it seems to be a useful tool in showing that \emph{efficient} online learnability is possible in various settings, often working in cases where more standard techniques appear to fail. Moreover, we hope the techniques we employ might prove useful in deriving practical online algorithms in other contexts.

\section{The Minimax Forecaster}
\label{sec:basic}

We start by formally introducing our online learning setting, known as prediction with expert advice (see \cite{CesaBianchiLu06}). The game is played between a forecaster and an adversary, and is specified by an outcome space $\Ycal$, a prediction space $\Pcal$, a nonnegative loss function $\ell : \Pcal\times\Ycal\to\reals$, which measures the discrepancy between the forecaster's prediction and the outcome, and an expert class $\Fcal$. Here we focus on classes $\Fcal$ of \textsl{static experts}, whose prediction at each round $t$ do not depend on the outcome in previous rounds. Therefore, we think of each $\boldf\in \Fcal$ simply as a sequence $\boldf = (f_1,f_2,\dots)$ where each $f_t\in\Pcal$. At each step $t=1,2,\dots$ of the game, the forecaster outputs a prediction $p_t\in\Pcal$ and simultaneously the adversary reveals an outcome $y_t\in\Ycal$. The forecaster's goal is to predict the outcome sequence almost as well as the best expert in the class $\Fcal$, irrespective of the outcome sequence $\by = (y_1,y_2,\dots)$. The performance of a forecasting strategy $A$ is measured by the worst-case regret
\begin{equation}
\label{eq:regret}
    \Vcal_T(A,\Fcal)
=
    \sup_{\by\in\Ycal^T} \left(\sum_{t=1}^{T}\ell(p_t,y_t) - \inf_{\boldf\in\Fcal}\sum_{t=1}^{T}\ell(f_t,y_t)\right)
\end{equation}
viewed as a function of the horizon (number of rounds) $T$.

Consider now the special case where the horizon $T$ is fixed and known in advance, the outcome space is $\Ycal=\{-1,+1\}$, the prediction space is $\Pcal=[-1,+1]$, and the loss is the absolute loss $\ell(p,y) = |p-y|$. To simplify notation, let
$L(\boldf,\by) = \sum_{t=1}^{T}|f_t-y_t|$. We will denote the regret in this special case as $\Vcal_T^{\abs}(A,\Fcal)$.

The Minimax Forecaster ---which is based on work presented in~\cite{CesabianchiFrHeHaScWa97} and~\cite{Chung94}, see also \cite{CesaBianchiLu06} for an exposition--- is derived by an explicit analysis of the minimax regret $\inf_A\Vcal_T^{\abs}(A,\Fcal)$,
where the infimum is over all forecasters $A$ producing at round $t$ a prediction $p_t$ as a function of $p_1,y_1,\ldots p_{t-1},y_{t-1}$.
For general online learning problems, the analysis of this quantity is intractable. However, for the specific setting we focus on (absolute loss and binary outcomes), one can get both an explicit expression for the minimax regret, as well as an explicit algorithm, provided $\inf_{\boldf\in \Fcal}\sum_{t=1}^{T}\ell(f_t,y_t)$ can be efficiently computed for any sequence $y_1,\ldots,y_T$. This procedure is akin to performing empirical risk minimization (ERM) in statistical learning. A full development of the analysis is out of scope, but is outlined in \secref{app:derivation}. In a nutshell, the idea is to begin by calculating the optimal prediction in the last round $T$, and then work backwards, calculating the optimal prediction at round $T-1$, $T-2$ etc. Remarkably, the value of $\inf_A\Vcal_T^{\abs}(A,\Fcal)$ is \emph{exactly} the Rademacher complexity $\Rcal_T(\Fcal)$ of the class $\Fcal$, which is known to play a crucial role in understanding the sample complexity in statistical learning~\cite{BartlettMe01}. In this paper, we define it as:
\begin{equation}\label{eq:rademachercomp}
    \Rcal_T(\Fcal) = \Exp\left[\sup_{\boldf\in\Fcal}\sum_{t=1}^{T}\sigma_t f_t\right]
\end{equation}
where $\sigma_1,\ldots,\sigma_T$ are i.i.d.\ Rademacher random variables, taking values $-1,+1$ with equal probability. When $\Rcal_T(\Fcal)=o(T)$, we get a minimax regret $\inf_A\Vcal_T^{\abs}(A,\Fcal)=o(T)$ which implies a vanishing per-round regret.

In terms of an explicit algorithm, the optimal prediction $p_t$ at round $t$ is given by a complicated-looking recursive expression, involving exponentially many terms. Indeed, for general online learning problems, this is the most one seems able to hope for. However, an apparently little-known fact is that when one deals with a class $\Fcal$ of fixed binary sequences as discussed above, then one can write the optimal prediction $p_t$ in a much simpler way. Letting $Y_1,\ldots,Y_T$ be i.i.d.\ Rademacher random variables, the optimal prediction at round $t$ can be written as
\begin{equation}
\label{eq:p_t}
p_t ~=~ \Exp\left[\inf_{\boldf\in\Fcal}L\left(\boldf,y_1\cdots y_{t-1}\,(-1)\,Y_{t+1}\cdots Y_T\right)-\inf_{\boldf\in\Fcal}L\left(\boldf,y_1\cdots y_{t-1}\,1\,Y_{t+1}\cdots Y_T\right)\right].
\end{equation}
In words, the prediction is simply the expected difference between the minimal cumulative loss over $\Fcal$, when the adversary plays $-1$ at round $t$ and random values afterwards, and the minimal cumulative loss over $\Fcal$, when the adversary plays $+1$ at round $t$, and the same random values afterwards. Again, we refer the reader to \secref{app:derivation} for how this is derived. We denote this optimal strategy (for absolute loss and binary outcomes) as the Minimax Forecaster (\textsc{mf}).
\begin{algorithm}[H]
\caption{Minimax Forecaster (\textsc{mf})}
\begin{algorithmic}
\FOR{$t=1$ to $T$}
    \STATE Predict $p_t$ as defined in \eqref{eq:p_t}
    \STATE Receive outcome $y_t$ and suffer loss $|p_t-y_t|$
\ENDFOR
\end{algorithmic}
\label{alg:minimax}
\end{algorithm}
The relevant guarantee for \textsc{mf} is summarized in the following theorem.
\begin{theorem}\label{thm:basic}
For any class $\Fcal\subseteq [-1,+1]^T$ of static experts,
the regret of the Minimax Forecaster (\algref{alg:minimax}) satisfies $\Vcal_T^{\abs}(\textsc{mf},\Fcal) = \Rcal_T(\Fcal)$.
\end{theorem}


The Minimax Forecaster described above is not computationally efficient, as the computation of $p_t$ requires averaging over exponentially many ERM's. However, by a martingale argument, it is not hard to show that it is in fact sufficient to compute only two ERM's per round.

\begin{algorithm}
\caption{Minimax Forecaster with efficient implementation (\textsc{mf*})}
\begin{algorithmic}
\FOR{$t=1$ to $T$}
    \STATE For $i=t+1,\ldots,T$, let $Y_i$ be a Rademacher random variable
    \STATE Let $p_t := \inf_{\boldf\in\Fcal}L\left(\boldf,y_1\dots y_{t-1}\,(-1)\,Y_{t+1}\dots Y_T\right)-\inf_{\boldf\in\Fcal}L\left(\boldf,y_1\dots y_{t-1}\,1\,Y_{t+1}\dots Y_T\right)$
    \STATE Predict $p_t$, receive outcome $y_t$ and suffer loss $|p_t-y_t|$
\ENDFOR
\end{algorithmic}
\label{alg:minimaxefficient}
\end{algorithm}

\begin{theorem}\label{thm:basicefficient}
For any class $\Fcal\subseteq [-1,+1]^T$ of static experts,
the regret of the randomized forecasting strategy \textsc{mf*} (\algref{alg:minimaxefficient}) satisfies
\[
\Vcal_T^{\abs}(\textsc{mf*},\Fcal) \le \Rcal_T(\Fcal)+\sqrt{2T\ln(1/\delta)}
\]
with probability at least $1-\delta$. Moreover, if the predictions $\bp = (p_1,\dots,p_T)$ are computed reusing the random values $Y_1,\ldots,Y_T$ computed at the first iteration of the algorithm, rather than drawing fresh values at each iteration, then it holds that
\[
    \Exp\left[L(\bp,\by) - \inf_{\boldf\in\Fcal} L(\boldf,\by)\right] \le \Rcal_T(\Fcal)
\qquad
    \text{for all $\by \in \{-1,+1\}^T$.}
\]
\end{theorem}
\begin{proof}[Proof sketch]
\smartqed
To prove
the second statement, note that $\bigl|\Exp[p_t]-y_t\bigr|=\Exp\bigl[|p_t-y_t|\bigr]$ for any fixed $y_t\in\{-1,+1\}$ and $p_t$ bounded in $[-1,+1]$, and use \thmref{thm:basic}.
To prove
the first statement, note that $|p_t-y_t|-\bigl|\Exp_{p_t}[p_t]-y_t\bigr|$ for $t=1,\dots,T$ is a martingale difference sequence with respect to $p_1,\ldots,p_T$, and apply Azuma's inequality.
\qed
\end{proof}
The second statement in the theorem bounds the regret only in expectation and is thus weaker than the first one. On the other hand, it might have algorithmic benefits. Indeed, if we reuse the same values for $Y_1,\ldots,Y_T$, then the computation of the infima over $\boldf$ in \textsc{mf*} are with respect to an outcome sequence which changes only at one point in each round. Depending on the specific learning problem, it might be easier to re-compute the infimum after changing a single point in the outcome sequence, as opposed to computing the infimum over a different outcome sequence in each round.

\section{The $R^2$ Forecaster}
\label{s:gen-minimax}
The Minimax Forecaster presented above is very specific to the absolute loss $\ell(f,y)=|f-y|$ and for binary outcomes $\Ycal=\{-1,+1\}$, which limits its applicability. We note that extending the forecaster to other losses or different outcome spaces is not trivial: indeed, the recursive unwinding of the minimax regret term, leading to an explicit expression and an explicit algorithm, does not work as-is for other cases. Nevertheless, we will now show how one can deal with general (convex, Lipschitz) loss functions and outcomes belonging to any real interval $[-b,b]$.

The algorithm we propose essentially uses the Minimax Forecaster as a subroutine, by feeding it with a carefully chosen sequence of binary values $z_t$, and using predictions $f_t$ which are scaled to lie in the interval $[-1,+1]$. The values of $z_t$ are based on a randomized rounding of values in $[-1,+1]$, which depend in turn on the loss subgradient. Thus, we denote the algorithm as the Randomized Rounding ($R^2$) Forecaster.

To describe the algorithm, we introduce some notation. For any scalar $f \in [-b,b]$, define
$
\widetilde{f} = f/b
$
to be the scaled versions of $f$ into the range $[-1,+1]$. For vectors $\boldf$, define $\widetilde{\boldf}=(1/b)\boldf$. Also, we let $\partial_{p_t} \ell(p_t,y_t)$ denote any subgradient of the loss function $\ell$ with respect to the prediction $p_t$. As before, we define $L(\widetilde{\boldf},\by) = \sum_{t=1}^{T}|\tilde{f}_t-y_t|$. The pseudocode of the $R^2$ Forecaster is presented as \algref{alg:genminimax} below, and its regret guarantee is summarized in \thmref{thm:main}.

\begin{algorithm}
\caption{The $R^2$ Forecaster}
\begin{algorithmic}
\STATE \textbf{Input:} Upper bound $b$ on $|f_t|,|y_t|$ for all $t=1,\dots,T$ and $\boldf\in\Fcal$; upper bound $\rho$ on $\sup_{p,y\in[-b,b]}\bigl|\partial_{p}\ell(p,y)\bigr|$; precision parameter $\eta\ge\tfrac{1}{T}$.
\FOR{$t=1$ to $T$}
    \STATE $p_t:=0$
    \FOR{$j=1$ to $\eta\,T$}
    \STATE For $i=t,\ldots,T$, let $Y_i$ be a Rademacher random variable
    \STATE
    Draw ${\dt \Delta := \inf_{\boldf\in\Fcal} L\left(\widetilde{\boldf},z_1\dots z_{t-1}\,(-1)\,Y_{t+1}\dots Y_T\right) - \inf_{\boldf\in\Fcal} L\left(\widetilde{\boldf},z_1\dots z_{t-1}\,1\,Y_{t+1}\dots Y_T\right)}$
    \STATE Let $p_t:=p_t+\frac{b}{\eta\,T}\Delta$
    \ENDFOR
    \STATE Predict $p_t$
    \STATE Receive outcome $y_t$ and suffer loss $\ell(p_t,y_t)$
    \STATE Let $r_t := \frac{1}{2}\bigl(1-\frac{1}{\rho}\partial_{p_t} \ell(p_t,y_t)\bigr) \in [0,1]$
    \STATE Let $z_t := 1$ with probability $r_t$, and $z_t:=-1$ with probability $1-r_t$
\ENDFOR
\end{algorithmic}
\label{alg:genminimax}
\end{algorithm}

\begin{theorem}\label{thm:main}
Suppose $\ell$ is convex and $\rho$-Lipschitz in its first argument. For any $\Fcal\subseteq[-b,b]^{T}$,
with probability at least $1-\delta$ the regret of the $R^2$ Forecaster (\algref{alg:genminimax}) satisfies
\begin{equation}\label{eq:as1}
    \Vcal_T(R^2,\Fcal)
\le
    \rho\,\mathcal{R}_T(\Fcal) + \rho\,b\left(\sqrt{\frac{1}{\eta}}+2\right) \sqrt{2T\ln\left(\frac{2T}{\delta}\right)}
\end{equation}
\end{theorem}
\begin{proof}
\smartqed
Let $Y(t)$ denote the set of Bernoulli random variables chosen at round $t$.
Let $\Exp_{z_t}$ denote expectation with respect to $z_t$,
conditioned on $z_{1},Y(1),\ldots,z_{t-1},Y(t-1)$ as well as $Y(t)$. Let $\Exp_{Y(t)}$ denote the expectation with respect to the random drawing of $Y(t)$, conditioned on $z_{1},Y(1),\ldots,z_{t-1},Y(t-1)$.

We will need two simple observations. First, by convexity of the loss function, we have that for any $p_t,f_t,y_t$,
$
\ell(p_t,y_t)-\ell(f_t,y_t)\leq (p_t-f_t)\,\partial_{p_t} \ell(p_t,y_t)
$.
Second, by definition of $r_t$ and $z_t$, we have that for any fixed $p_t,f_t$,
\begin{align*}
\frac{1}{\rho b}(p_t-f_t)&\partial_{p_t} \ell(p_t,y_t)
~=~ \frac{1}{b}(p_t-f_t)(1-2r_t)\\
&=~ \frac{1}{b}r_t(f_t-p_t)+\frac{1}{b}(1-r_t)(p_t-f_t)\\
&=~ r_t(\widetilde{f}_t-\widetilde{p}_t)+(1-r_t)(\widetilde{p}_t-\widetilde{f}_t)\\
&=~ r_t\left(\left(1-\widetilde{p}_t\right)
-\left(1-\widetilde{f}_t\right)\right)
+(1-r_t)\left(\left(\widetilde{p}_t+1\right)-\left(\widetilde{f}_t+1
\right)\right)\\
&=~ \Exp_{z_t}\left[\left|\widetilde{p}_t-z_t\right|
-\left|\widetilde{f}_t-z_t\right|\right]~.
\end{align*}
The last transition uses the fact that $\widetilde{p}_t,\widetilde{f}_t\in [-1,+1]$. By these two observations, we have
\begin{align}
    \sum_{t=1}^{T} \left(\ell(p_t,y_t) - \ell(f_t,y_t)\right)
\le
    \sum_{t=1}^{T}(p_t-f_t)\,\partial_{p_t} \ell(p_t,y_t)
=
    \rho\,b~\sum_{t=1}^{T}\Exp_{z_t}\left[ \left|\widetilde{p}_t-z_t\right|-\left|\widetilde{f}_t-z_t\right|\right]~.
\label{eq:2ob}
\end{align}
Now, note that
$|\widetilde{p}_t-z_t|-|\widetilde{f}_t-z_t| - \Exp_{z_t}\bigl[|\widetilde{p}_t-z_t|-|\widetilde{f}_t-z_t|\bigr]$ for $t=1,\dots,T$
is a martingale difference sequence: for any values of $z_1,Y(1),\ldots,z_{t-1},Y(t-1),Y(t)$ (which fixes $\widetilde{p}_t$), the conditional expectation of this expression over $z_t$ is zero. Using Azuma's inequality, we can upper bound \eqref{eq:2ob} with probability at least $1-\delta/2$ by
\begin{equation}\label{eq:2ob1}
\rho\,b~\sum_{t=1}^{T}\left(\left|\widetilde{p}_t-z_t\right|-|\widetilde{f}_t-z_t|\right)+\rho\,b\sqrt{8T\ln(2/\delta)}.
\end{equation}
The next step is to relate \eqref{eq:2ob1} to $\rho\,b\sum_{t=1}^{T}\bigl(\left|\Exp_{Y(t)}[\widetilde{p}_t]-z_t\right|-|\widetilde{f}_t-z_t|\bigr)$. It might be tempting to appeal to Azuma's inequality again. Unfortunately, there is no martingale difference sequence here, since $z_t$ is itself a random variable whose distribution is influenced by $Y(t)$. Thus, we need to turn to coarser methods. \eqref{eq:2ob1} can be upper bounded by
\begin{equation}
\rho\,b~\sum_{t=1}^{T}\left(\left|\Exp_{Y(t)}[\widetilde{p}_t]-z_t\right|-
|\widetilde{f}_t-z_t|\right)
+\rho\,b~\sum_{t=1}^{T}\left|\widetilde{p}_t-\Exp_{Y(t)}[\widetilde{p}_t]\right|
+\rho\,b\sqrt{8T\ln(2/\delta)}.\label{eq:2ob2}
\end{equation}
Recall that $\widetilde{p}_t$ is an average over $\eta T$ i.i.d.\ random variables, with expectation $\Exp_{Y(t)}[\widetilde{p}_t]$.
By Hoeffding's inequality, this implies that for any $t=1,\dots,T$,
with probability at least $1-\delta/2T$ over the choice of $Y(t)$,
$
\left|\widetilde{p}_t-\Exp_{Y(t)}[\widetilde{p}_t]\right| \leq \sqrt{2{\ln(2T/\delta)}\big/{(\eta T)}}.
$
By a union bound, it follows that with probability at least $1-\delta/2$ over the choice of $Y(1),\ldots,Y(T)$,
\[
\sum_{t=1}^{T}\left|\widetilde{p}_t-\Exp_{Y(t)}[\widetilde{p}_t]\right| \leq \sqrt{\frac{2T\ln(2T/\delta)}{\eta}}~.
\]
Combining this with \eqref{eq:2ob2}, we get that with probability at least $1-\delta$,
\begin{equation}\label{eq:2ob3}
\rho\,b\sum_{t=1}^{T}\left(\left|\Exp_{Y(t)}[\widetilde{p}_t]-z_t\right|-
|\widetilde{f}_t-z_t|\right)
+\rho\,b\sqrt{\frac{2T\ln(2T/\delta)}{\eta}}+\rho\,b\sqrt{8T\ln(2/\delta)}~.
\end{equation}
Finally, by definition of $\widetilde{p}_t = p_t/b$, we have that $\Exp_{Y(t)}[\widetilde{p}_t]$ equals
\[
    \Exp_{Y(t)}\!\!\left[\inf_{\boldf\in\Fcal}
L\left(\widetilde{\boldf},z_1\dots z_{t-1}\,(-1)\,Y_{t+1}\dots Y_T\right) - \inf_{\boldf\in\Fcal}L\left(\widetilde{\boldf},z_1\dots z_{t-1}\,1\,Y_{t+1}\dots Y_T\right)\right].
\]
\normalsize
This is exactly the Minimax Forecaster's prediction at round $t$, with respect to the sequence of outcomes $z_1,\ldots,z_{t-1}\in\{-1,+1\}$, and the class $\widetilde{\Fcal}:=\bigl\{\widetilde{\boldf}:\boldf\in \Fcal\bigr\}\subseteq [-1,1]^T$. Therefore, using \thmref{thm:basic}, we can upper bound \eqref{eq:2ob3} by
\[
\rho\,b\,\Rcal_T(\widetilde{\Fcal})+\rho\,b\sqrt{\frac{2T\ln(2T/\delta)}{\eta}}+\rho\,b\sqrt{8T\ln(2/\delta)}~.
\]
By definition of $\widetilde{\Fcal}$ and Rademacher complexity, it is straightforward to verify that $\Rcal_T(\widetilde{\Fcal})=\frac{1}{b}\Rcal_T(\Fcal)$. Using that to rewrite the bound, and slightly simplifying for readability, the result stated in the theorem follows.
\qed
\end{proof}
The computed prediction $p_t$ is an empirical approximation to
\[
b\,\Exp_{Y_{t+1},\ldots,Y_T}\left[\inf_{\boldf\in\Fcal} L\left(\widetilde{\boldf},z_1\dots z_{t-1}\,0\,Y_{t+1}\dots Y_T\right) - \inf_{\boldf\in\Fcal} L\left(\widetilde{\boldf},z_1\cdots z_{t-1}\,1\,Y_{t+1}\cdots Y_T\right)\right]
\]
by repeatedly drawing independent values to $Y_{t+1},\ldots,Y_T$ and averaging. The accuracy of the approximation is reflected in the precision parameter $\eta$. A larger value of $\eta$ improves the regret bound, but also increases the runtime of the algorithm. Thus, $\eta$ provides a trade-off between the computational complexity of the algorithm and its regret guarantee. We note that even when $\eta$ is taken to be a constant fraction, the resulting algorithm still runs in polynomial time $\Ocal(T^2c)$, where $c$ is the time to compute a single ERM. In subsequent results pertaining to this Forecaster, we will assume that $\eta$ is taken to be a constant fraction.

The $R^2$ forecaster, as presented so far, assumes that the horizon $T$ is known in advance. We now turn to describe how it can be readily extended to the case where it is unknown. The standard generic method to achieve this is known as the ``doubling'' trick (see \cite{CesaBianchiLu06}), and is based on guessing the value of $T$ (initially $T=1$), and running the algorithm with this guess. If the game did not end after $T$ rounds, the guess is doubled and the algorithm is restarted with this new value. If the actual horizon $T$ equals $2^0+2^1+2^2+\ldots+2^r$ for some integer $r$, then it is easy to show that our algorithm enjoys the same regret bound as before, plus a moderate multiplicative factor\footnote{Specifically, we divide the rounds into $r$ consecutive epochs, such that epoch $i$ consists of $2^i$ rounds, and use \thmref{thm:main} with confidence $\delta'=\delta/2^{i+1}$, and a union bound, to get a regret bound of $\Ocal(\mathcal{R}_{2^i}(\Fcal)+\sqrt{\left(i+\log(1/\delta)\right)2^i})$ over any epoch $i$. In the typical case where $\mathcal{R}_{T}(\Fcal) = \Ocal(\sqrt{T})$, summing over $i=1,\ldots,r$ where $r=\log_2(T+1)-1$ yields a total regret bound of order $\Ocal(\sqrt{\log(T/\delta)T})$. Up to log factors, this is the same bound as if $T$ were known in advance.}. The only case we need to worry about is when $T$ is not of this form, i.e., that the game ends in the middle of the algorithm's run. In that case, it is enough to ensure that the algorithm's regret bound, designed for horizon $T$, also bounds the regret after a smaller number $t<T$ of rounds. This can be shown to hold quite generically, given a very mild assumption on the loss function:
\begin{lemma}
\label{lem:T-uniform}
Consider a (possibly randomized) forecaster $A$ for a class $\Fcal$ whose regret after $T$ steps satisfies $\Vcal_T(A,\Fcal) \le G$ with probability at least $1-\delta > \tfrac{1}{2}$. Furthermore, suppose the loss function is such that ${\dt \inf_{p'\in\Pcal}\sup_{y\in\Ycal}\inf_{p\in\Pcal}\bigl(\ell(p,y)-\ell(p',y)\bigr) \geq 0 }$.
Then
\[
    \max_{t=1,\dots,T} \Vcal_t(A,\Fcal) \le G
\qquad
    \text{with probability at least $1-\delta$.}
\]
\end{lemma}
Note that for the assumption on the loss to hold, a simple sufficient condition is that $\Pcal=\Ycal$ and $\ell(p,y) \ge \ell(y,y)$ for all $p,y\in\Pcal$.
\begin{proof}
\smartqed
The proof assumes that the infimum and supremum of certain functions over $\Ycal,\Fcal$ are attainable. If not, the proof can be easily adapted by finding attainable values which are $\epsilon$-close to the infimum or supremum, and then taking $\epsilon \rightarrow 0$.

For the purpose of contradiction, suppose there exists a strategy for the adversary and a round $r \le T$ such that at the end of round $r$, the forecaster suffers a regret $G' > G$ with probability larger than $\delta$. Consider the following modified strategy for the adversary: the adversary plays according to the aforementioned strategy until round $r$. It then computes
\[
f^* = \argmin{f\in\Fcal}\sum_{t=1}^{r}\ell(f_t,y_t)~.
\]
At all subsequent rounds $t=r+1,r+2,\ldots,T$, the adversary chooses
\[
    y^*_t = \argmax{y\in\Ycal}\inf_{p\in\Pcal}\bigl(\ell(p,y)-\ell(f^*_t,y)\bigr)~.
\]
By the assumption on the loss function,
\[
\ell(p_t,y^*_t)-\ell(f^*_t,y^*_t) \geq \inf_{p\in\Pcal} \bigl(\ell(p,y^*_t)-\ell(f^*_t,y^*_t)\bigr) = \sup_{y\in \Ycal}\inf_{p\in\Pcal} \bigl(\ell(p,y)-\ell(f^*_t,y)\bigr) \geq 0~.
\]
Thus, the regret over all $T$ rounds, with respect to $f^*$, is
\[
\sum_{t=1}^{r}\bigl(\ell(p_t,y_t)-\ell(f^*_t,y_t)\bigr)+\sum_{t=r+1}^{T}\bigl(\ell(p_t,y^*_t)-\ell(f^*_t,y^*_t)\bigr)
\geq \sum_{t=1}^{r}\ell(p_t,y_t)-\inf_{f\in\Fcal}\sum_{t=1}^{r}\ell(f_t,y_t)
\]
which is at least $G'$ with probability larger than $\delta$. On the other hand, we know that the learner's regret is at most most $G$ with probability at least $1-\delta$. Thus we have a contradiction and the proof is concluded.
\qed
\end{proof}
We end this section with a remark that plays an important role in what follows.
\begin{remark}
\label{rem:perm}
The predictions of our forecasting strategies do not depend on the ordering of the predictions of the experts in $\Fcal$. In other words, all the results proven so far also hold in a setting where the elements of $\Fcal$ are functions $f : \{1,\ldots,T\}\to\Pcal$, and the adversary has control on the permutation $\pi_1,\dots,\pi_T$ of $\{1,\ldots,T\}$ that is used to define the prediction $f(\pi_t)$ of expert $f$ at time $t$.\footnote{
Formally, at each step $t$: (1) the adversary chooses and reveals the next element $\pi_t$ of the permutation; (2) the forecaster chooses $p_t\in\Pcal$ and simultaneously the adversary chooses $y_t\in\Ycal$.
}
Also, \thmref{thm:basic} implies that the value of $\Vcal_T^{\abs}(\Fcal)$ remains unchanged irrespective of the permutation chosen by the adversary.
\end{remark}

\section{Application 1: Transductive Online Learning}
\label{s:trans}
The first application we consider is a rather straightforward one, in the context of transductive online learning \cite{BendKuMa97}. In this model, we have an arbitrary sequence of labeled examples $(x_1,y_1),\ldots,(x_T,y_T)$, where only the set $\{x_1,\ldots,x_T\}$ of unlabeled instances is known to the learner in advance. At each round $t$, the learner must provide a prediction $p_t$ for the label of $y_t$. The true label $y_t$ is then revealed, and the learner incurs a loss $\ell(p_t,y_t)$. The learner's goal is to minimize the transductive online regret $\sum_{t=1}^{T}\bigl(\ell(p_t,y_t)-\inf_{f\in\Fcal}\ell(f(x_t),y_t)\bigr)$ with respect to a fixed class of predictors $\Fcal$ of the form $\{x\mapsto f(x)\}$.

The work~\cite{KakadeKalai05} considers the binary classification case with zero-one loss. Their main result is that if a class $\Fcal$ of binary functions has bounded VC dimension $d$, and there exists an efficient algorithm to perform empirical risk minimization, then one can construct an efficient randomized algorithm for transductive online learning, whose regret is at most $\Ocal(T^{3/4}\sqrt{d\ln(T)})$ in expectation. The significance of this result is that efficient batch learning (via empirical risk minimization) implies efficient learning in the transductive online setting. This is an important result, as online learning can be computationally harder than batch learning - see, e.g., \cite{Blum94} for an example in the context of Boolean learning.

A major open question posed by~\cite{KakadeKalai05} was whether one can achieve the optimal rate $\Ocal(\sqrt{dT})$, matching the rate of a batch learning algorithm in the statistical setting. Using the $R^2$ Forecaster, we can easily achieve the above result, as well as similar results in a strictly more general setting. This shows that efficient batch learning not only implies efficient transductive online learning (the main thesis of \cite{KakadeKalai05}), but also that the same rates can be obtained, and for possibly non-binary prediction problems as well.
\begin{theorem}
Suppose we have a computationally efficient algorithm for empirical risk minimization (with respect to the zero-one loss) over a class $\Fcal$ of $\{0,1\}$-valued functions with VC dimension $d$. Then, in the transductive online model, the efficient randomized forecaster \textsc{mf*} achieves an expected regret of $\Ocal(\sqrt{dT})$ with respect to the zero-one loss. \\
Moreover, for an arbitrary class $\Fcal$ of $[-b,b]$-valued functions with Rademacher complexity $\Rcal_T(\Fcal)$, and any convex $\rho$-Lipschitz loss function, if there exists a computationally efficient algorithm for empirical risk minimization, then the $R^2$ Forecaster is computationally efficient and achieves, in the transductive online model, a regret of $\rho\Rcal_T(\Fcal)+\Ocal(\rho b \sqrt{T\ln(T/\delta)})$ with probability at least $1-\delta$.
\end{theorem}

\begin{proof}
\smartqed
Since the set $\{x_1,\ldots,x_T\}$ of unlabeled examples is known, we reduce the online transductive model to prediction with expert advice in the setting of Remark~\ref{rem:perm}. This is done by mapping each function $f\in \Fcal$ to a function $f : \{1,\ldots,T\} \to \Pcal$ by $t\mapsto f(x_t)$, which is equivalent to an expert in the setting of Remarks~\ref{rem:perm}. When $\Fcal$ maps to $\{0,1\}$, and we care about the zero-one loss, we can use the forecaster \textsc{mf*} to compute randomized predictions and apply \thmref{thm:basicefficient} to bound the expected transductive online regret  with $\Rcal_T(\Fcal)$. For a class with VC dimension $d$, $\Rcal_T(\Fcal) \le \Ocal(\sqrt{dT})$ for some constant $c > 0$, using Dudley's chaining method \cite{Dudley84}, and this concludes the proof of the first part of the theorem. The second part is an immediate corollary of \thmref{thm:main}.
\qed
\end{proof}
We close this section by contrasting our results for online transductive learning with those of~\cite{BenDavidPS09} about standard online learning. If $\Fcal$ contains $\{0,1\}$-valued functions, then the optimal regret bound for online learning is order of $\sqrt{d'T}$, where $d'$ is the Littlestone dimension of $\Fcal$. Since the Littlestone dimension of a class is never smaller than its VC dimension, we conclude that online learning is a harder setting than online transductive learning.

\section{Application 2: Online Collaborative Filtering}
\label{s:ocofilt}

We now turn to discuss the application of our results in the context of collaborative filtering with trace-norm constrained matrices, presenting the first computationally efficient online algorithms for this problem.

In collaborative filtering, the learning problem is to predict entries of an unknown $m \times n$ matrix based on a subset of its observed entries. A common approach is norm regularization, where we seek a low-norm matrix which matches the observed entries as best as possible. The norm is often taken to be the trace-norm~\cite{Bach08,SalaMni07,SrebRenJaa04}, although other norms have also been considered, such as the max-norm~\cite{LeeReSaSerTrp10} and the weighted trace-norm~\cite{FoSaShaSre11,SalSre10}.

Previous theoretical treatments of this problem assumed a stochastic setting, where the observed entries are picked according to some underlying distribution (e.g., \cite{ShalSham11,SreSh05}). However, even when the guarantees are distribution-free, assuming a fixed distribution fails to capture important aspects of collaborative filtering in practice, such as non-stationarity~\cite{Koren09}. Thus, an online adversarial setting, where no distributional assumptions whatsoever are required, seems to be particularly well-suited to this problem domain.

In an online setting, at each round $t$ the adversary reveals an index pair $(i_t,j_t)$ and secretely chooses a value $y_t$ for the corresponding matrix entry. After that, the learner selects a prediction $p_t$ for that entry. Then $y_t$ is revealed and the learner suffers a loss $\ell(p_t,y_t)$. Hence, the goal of a learner is to minimize the regret with respect to a fixed class $\Wcal$ of prediction matrices,
$
    \sum_{t=1}^T \ell(p_t,y_t) - \inf_{W\in\Wcal} \sum_{t=1}^T \ell\bigl(W_{i_t,j_t},y_t\bigr)
$.
Following reality, we will assume that the adversary picks a different entry in each round.
When the learner's performance is measured by the regret after all $T=mn$ entries have been predicted,
the online collaborative filtering setting reduces to prediction with expert advice as discussed in Remark~\ref{rem:perm}.

As mentioned previously, $\Wcal$ is often taken to be a convex class of matrices with bounded trace-norm. Many convex learning problems, such as linear and kernel-based predictors, as well as matrix-based predictors, can be learned efficiently both in a stochastic and an online setting, using mirror descent or regularized follow-the-leader methods. However, for reasonable choices of $\Wcal$, a straightforward application of these techniques leads to algorithms with trivial bounds. In particular, in the case of $\Wcal$ consisting of $m\times n$ matrices with trace-norm at most $r$, standard online regret bounds would scale like $\Ocal\bigl(r\sqrt{T}\bigr)$. Since for this norm one typically has $r = \Ocal\bigl(\sqrt{mn}\bigr)$, we get a per-round regret guarantee of $\Ocal(\sqrt{mn/T})$. This is a trivial bound, since it becomes ``meaningful'' (smaller than a constant) only after all $T=mn$ entries have been predicted. In this section, we show how to obtain a computationally efficient algorithm for this problem, using the $R^2$ Forecaster. We note that following our work, other efficient algorithms were proposed in \cite{HazKaSha12,RaShaSri12}.

Consider first the transductive online setting, where the set of indices to be predicted is known in advance, and the adversary may only choose the order and values of the entries. It is readily seen that the $R^2$ Forecaster can be applied in this setting, using any convex class $\Wcal$ of fixed matrices with bounded entries to compete against, and any convex Lipschitz loss function. To do so, we let $\{i_k,j_k\}_{k=1}^{T}$ be the set of entries, and run the $R^2$ Forecaster with respect to $\Fcal=\{t\mapsto W_{i_t,j_t}~:~W\in \Wcal\}$, which corresponds to a class of experts as discussed in Remark~\ref{rem:perm}.

What is perhaps more surprising is that the $R^2$ Forecaster can also be applied in a \emph{non-transductive} setting, where the indices to be predicted are not known in advance. Moreover, the Forecaster doesn't need to know the horizon $T$ in advance. The key idea to achieve this is to utilize the non-asymptotic nature of the learning problem ---namely, that the game is played over a finite $m\times n$ matrix, so the time horizon is necessarily bounded.

The algorithm we propose is very simple: we apply the $R^2$ Forecaster as if we are in a setting with time horizon $T=mn$, which is played over \emph{all} entries of the $m\times n$ matrix. By Remark \ref{rem:perm}, the $R^2$ Forecaster does not need to know the order in which these $m\times n$ entries are going to be revealed. Whenever $\Wcal$ is convex and $\ell$ is a convex function, we can find an ERM in polynomial time by solving a convex problem. Hence, we can implement the $R^2$ Forecaster efficiently.

Using \lemref{lem:T-uniform}, the following theorem exemplifies how we can obtain a regret guarantee for our algorithm, in the case of $\Wcal$ consisting of the convex set of matrices with bounded trace-norm and bounded entries. For the sake of clarity, we will consider $n\times n$ square matrices.
\begin{theorem}\label{thm:trace}
Let $\ell$ be a loss function which satisfies the conditions of \lemref{lem:T-uniform}. Also, let $\Wcal$ consist of $n \times n$ matrices with trace-norm at most $r = \Ocal(n)$ and entries at most $b=\Ocal(1)$, suppose we apply the $R^2$ Forecaster over time horizon $n^2$ and all entries of the matrix. Then with probability at least $1-\delta$, after $T$ rounds, the algorithm achieves an average per-round regret of at most
\[
    \Ocal\left(\frac{n^{3/2}+n\sqrt{\ln(n/\delta)}}{T}\right)
\qquad
    \text{uniformly over $T=1,\dots,n^2$.}
\]
\end{theorem}
\begin{proof}
\smartqed
In our setting, where the adversary chooses a different entry at each round, \cite[Theorem 6]{ShalSham11} implies that for the class $\Wcal'$ of all matrices with trace-norm at most $r=\Ocal(n)$, it holds that $\Rcal_{T}(\Wcal')/T \leq \Ocal(n^{3/2}/T)$. Therefore, $\Rcal_{n^2}(\Wcal')\leq \Ocal(n^{3/2})$. Since $\Wcal\subseteq \Wcal'$, we get by definition of the Rademacher complexity that $\Rcal_{n^2}(\Wcal)=\Ocal(n^{3/2})$ as well. By \thmref{thm:main}, the regret after $n^2$ rounds is $\Ocal(n^{3/2}+n\sqrt{\ln(n/\delta)})$ with probability at least $1-\delta$. Applying \lemref{lem:T-uniform}, we get that the cumulative regret at the end of any round $T=1,\ldots,n^2$ is at most $\Ocal(n^{3/2}+n\sqrt{\ln(n/\delta)})$, as required.
\qed
\end{proof}
This bound becomes non-trivial after $n^{3/2}$ entries are revealed, which is still a vanishing proportion of all $n^2$ entries. While the regret might seem unusual compared to standard regret bounds (which usually have rates of $1/\sqrt{T}$ for general losses), it is a natural outcome of the non-asymptotic nature of our setting, where $T$ can never be larger than $n^2$. In fact, this is the same rate one would obtain in a batch setting, where the entries are drawn from an arbitrary distribution.

As mentioned in the introduction, other online learning algorithms for this problem have been published since this work appeared \cite{HazKaSha12,RaShaSri12}, using other techniques and assumptions.

\section{Appendix: Derivation of the Minimax Forecaster}\label{app:derivation}

In this appendix, we outline how the Minimax Forecaster is derived, as well as its associated guarantees. This outline closely follows the exposition in \cite[Chapter 8]{CesaBianchiLu06}, to which we refer the reader for some of the technical derivations.

First, we note that the Minimax Forecaster as presented in \cite{CesaBianchiLu06} actually refers to a slightly different setup than ours, where the outcome space is $\Ycal=\{0,1\}$ and the prediction space is $\Pcal= [0,1]$, rather than $\Ycal=\{-1,+1\}$ and $\Pcal=[-1,+1]$. We will first derive the forecaster for the first setting, and then show how to convert it to the second setting.

Our goal is to find a predictor which minimizes the worst-case regret,
\[
\max_{\by\in \{0,1\}^T} \left(L(\bp,\by)-\inf_{\boldf\in \Fcal}L(\boldf,\by)\right)
\]
where $\bp = (p_1,\dots,p_T)$ is the prediction sequence.

For convenience, in the following we sometimes use the notation $\by^{t}$ to denote a vector in $\{0,1\}^t$.
The idea of the derivation is to work backwards, starting with computing the optimal prediction at the last round $T$, then deriving the optimal prediction at round $T-1$ and so on. In the last round $T$, the first $T-1$ outcomes $\by^{T-1}$ have been revealed, and we want to find the optimal prediction $p_T$. Since our goal is to minimize worst-case regret with respect to the absolute loss, we just need to compute $p_T$ which minimizes
\[
L(\bp^{T-1},\by^{T-1})+\max\Bigl\{p_T-\inf_{\boldf\in \Fcal}L(\boldf,\by^{T-1}0)
~,~(1-p_T)-\inf_{\boldf\in \Fcal}L(\boldf,\by^{T-1}1)\Bigr\}~.
\]
In our setting, it is not hard to show that $\bigl|\inf_{\boldf\in \Fcal}L(\boldf,\by^{t-1}0) - \inf_{\boldf\in \Fcal}L(\boldf,\by^{t-1}1)\bigr| \leq 1$ (see \cite[Lemma 8.1]{CesaBianchiLu06}). Using this, we can compute the optimal $p_T$ to be
\begin{equation}\label{eq:plast}
p_T = \frac{1}{2}\Bigl(A_T(\by^{T-1}1)-A_T(\by^{T-1}0)+1\Bigr)
\end{equation}
where $A_{T}(\by^T) = -\inf_{\boldf\in \Fcal}L(\boldf,\by^T)$.

Having determined $p_T$, we can continue to the previous prediction $p_{T-1}$. This is equivalent to minimizing
\[
L(\bp^{T-2},\by^{T-2})+\max\Bigl\{p_{T-1}+A_{T-1}(\by^{T-2}0)~,~ (1-p_{T-1})+A_{T-1}(\by^{T-2}1)\Bigr\}
\]
where
\begin{equation}\label{eq:Adef}
A_{T-1}(\by^{T-1}) = \min_{p_T\in [0,1]}\max\left\{p_T-\inf_{\boldf\in \Fcal} L(\boldf,\by^{T-1}0)~,~ (1-p_T)-\inf_{\boldf\in \Fcal}L(\boldf,\by^{T-1}1)\right\}.
\end{equation}
Note that by plugging in the value of $p_T$ from \eqref{eq:plast}, we also get the following equivalent formulation for $A_{T-1}(\by^{T-1})$:
\[
A_{T-1}(\by^{T-1}) = \frac{1}{2}\Bigl(A_T(\by^{T-1}0)+A_T(\by^{T-1}1)+1\Bigr).
\]
Again, it is possible to show that the optimal value of $p_{T-1}$ is
\[
p_{T-1} = \frac{1}{2}\Bigl(A_{T-1}(\by^{T-2}1)-A_T(\by^{T-2}0)+1\Bigr).
\]
Repeating this procedure, one can show that at any round $t$, the minimax optimal prediction is
\begin{equation}\label{eq:ptopt}
p_{t} = \frac{1}{2}\Bigl(A_{t}(\by^{t-1}1)-A_t(\by^{t-1}0)+1\Bigr)
\end{equation}
where $A_{t}$ is defined recursively as $A_T(\by^T) = -\inf_{\boldf\in \Fcal} L(\boldf,\by^{T})$ and, for all $t$,
\begin{equation}\label{eq:Adefrecursive}
A_{t-1}(\by^{t-1}) = \frac{1}{2}\Bigl(A_t(\by^{t-1}0)+A_t(\by^{t-1}1)+1\Bigr).
\end{equation}
At first glance, computing $p_t$ from \eqref{eq:ptopt} might seem tricky, since it requires computing $A_{t}(\by^t)$ whose recursive expansion in \eqref{eq:Adefrecursive} involves exponentially many terms. Luckily, the recursive expansion has a simple structure, and it is not hard to show that
\begin{equation}\label{eq:Aexplicit}
A_{t}(\by^{t}) ~=~\frac{T-t}{2}-\frac{1}{2^T}\sum_{\by\in\{0,1\}^T}\left(\inf_{\boldf\in \Fcal} L(\boldf,\by^{t}Y^{T-t})\right)
~=~ \frac{T-t}{2}-\Exp\Bigl[\inf_{\boldf\in \Fcal} L(\boldf,\by^{t}Y^{T-t})\Bigr]
\end{equation}
where $Y^{T-t}$ is a sequence of $T-t$ i.i.d.\ Bernoulli random variables, which take values in $\{0,1\}$ with equal probability. Plugging this into the formula for the minimax prediction in \eqref{eq:ptopt}, we get that\footnote{This fact appears in an implicit form in \cite{CesabianchiFrHeHaScWa97} ---see also \cite[Exercise 8.4]{CesaBianchiLu06}.}
\begin{equation}\label{eq:p_t01}
p_t = \frac{1}{2}\left(\Exp\left[\inf_{\boldf\in \Fcal} L(\boldf,\by^{t-1}0 Y^{T-t})-\inf_{\boldf\in \Fcal} L(\boldf,\by^{t-1}1 Y^{T-t})\right]+1\right).
\end{equation}
This prediction rule constitutes the Minimax Forecaster as presented in \cite{CesaBianchiLu06}.

After deriving the algorithm, we turn to analyze its regret performance. To do so, we just need to note that $A_0$ equals the worst-case regret ---see the recursive definition at \eqref{eq:Adef}. Using the alternative explicit definition in \eqref{eq:Aexplicit}, we get that the worst-case regret equals
\begin{align*}
\frac{T}{2}-\Exp\left[\inf_{\boldf\in\Fcal}\sum_{t=1}^{T}|f_t-Y_t|\right]
= \Exp\left[\sup_{\boldf\in\Fcal}\sum_{t=1}^{T}\left(\frac{1}{2}-|f_t-Y_t|\right)
\right]
= \Exp\left[\sup_{\boldf\in\Fcal}\sum_{t=1}^{T}\left(f_t-\frac{1}{2}\right)\sigma_t
\right]
\end{align*}
where $\sigma_t$ are i.i.d.\ Rademacher random variables (taking values of $-1$ and $+1$ with equal probability). Recalling the definition of Rademacher complexity, \eqref{eq:rademachercomp}, we get that the regret is bounded by the Rademacher complexity of the shifted class, which is obtained from $\Fcal$ by taking every $\boldf\in \Fcal$ and replacing every coordinate $f_t$ by $f_t-1/2$.

Finally, it remains to show how to convert the forecaster and analysis above to the setting discussed in this paper, where the outcomes are in $\{-1,+1\}$ rather than $\{0,1\}$ and the predictions are in $[-1,+1]$ rather than $[0,1]$. To do so, consider a learning problem in this new setting, with some class $\Fcal$. For any vector $\by$, define $\widetilde{\by}$ to be the shifted vector $(\by+\mathbf{1})/2$, where $\mathbf{1} = (1,\dots,1)$ is the all-ones vector. Also, define $\widetilde{\Fcal}$ to be the shifted class $\widetilde{\Fcal} = \{(\boldf+\mathbf{1})/2~:~\boldf\in\Fcal\}$. It is easily seen that $L(\boldf,\by) = 2L(\widetilde{\boldf},\widetilde{\by})$ for any $\boldf,\by$. As a result, if we look at the prediction $p_t$ given by our forecaster in \eqref{eq:p_t}, then $\widetilde{p_t} = (p_t+1)/2$ is the minimax optimal prediction given by \eqref{eq:p_t01} with respect to the class $\widetilde{\Fcal}$ and the outcomes $\widetilde{\by}^T$. So our analysis above applies, and we get that
\begin{align*}
\max_{\by \in \{-1,+1\}^T} &\left(L(\bp,\by)-\inf_{\boldf\in \Fcal}L(\boldf,\by)\right)
~=~
\max_{\widetilde{\by} \in [0,1]^T} 2\left(L(\widetilde{\bp},\widetilde{\by})-\inf_{\widetilde{\boldf}\in \widetilde{\Fcal}}L(\widetilde{\boldf},\widetilde{\by})\right)
\\ &=~
2\Exp\left[\sup_{\widetilde{\boldf}\in\widetilde{\Fcal}}\sum_{t=1}^{T}\left(\widetilde{f}_t-\frac{1}{2}\right)\sigma_t \right]
~=~
\Exp\left[\sup_{\boldf\in\Fcal}\sum_{t=1}^{T}\sigma_t f_t \right]
\end{align*}
which is exactly the Rademacher complexity of the class $\Fcal$.

\textbf{Acknowledgements:} The first author acknowledges partial support by the PASCAL2 NoE under EC grant FP7-216886.

\bibliographystyle{plain}
\bibliography{mybib}

\end{document}